\documentclass[12pt]{article}

 \usepackage{authblk}

 \usepackage{amsfonts}
 \usepackage{amsmath}
 \usepackage{amssymb}
 \usepackage{amsthm}

 \usepackage{graphicx}
 \usepackage[update,prepend]{epstopdf}

 \usepackage{booktabs} 

 \usepackage{caption}

 \usepackage{multirow}

 \usepackage{float}

\usepackage{subcaption} 

\newcommand{\E}{\text{E}}

\newcommand{\M}{\mathcal{M}}

\newcommand{\pr}{\text{Pr}}

\newcommand{\R}{\mathbb{R}}

\renewcommand{\l}{\left}
\renewcommand{\r}{\right}

\renewcommand{\l}{\left}
\renewcommand{\r}{\right}

\newcommand{\tr}{\text{Tr}}

\newcommand{\T}{\mathcal{T}}

\newtheorem{proposition}{Proposition}
\newtheorem{theorem}{Theorem}
\newtheorem{lemma}{Lemma}
\newtheorem{corollary}{Corollary}
\title{Linear Dimensionality Reduction in Linear Time: Johnson-Lindenstrauss-type Guarantees for Random Subspace}
\author[1]{Nick Lim\thanks{jsl18@students.waikato.ac.nz}}
\author[2]{Robert J. Durrant\thanks{bobd@waikato.ac.nz}}
\affil[ ]{Department of Mathematics and Statistics}
\affil[ ]{University of Waikato}
\affil[ ]{Hamilton, New Zealand}

\date{} 
\begin{document}

\maketitle

\begin{abstract}
We consider the problem of efficient randomized dimensionality reduction with norm-preservation guarantees. Specifically we prove data-dependent Johnson-Lindenstrauss-type geometry preservation guarantees for Ho's random subspace method: When data satisfy a mild regularity condition -- the extent of which can be estimated by sampling from the data -- then random subspace approximately preserves the Euclidean geometry of the data with high probability. Our guarantees are of the same order as those for random projection, namely the required dimension for projection is logarithmic in the number of data points, but have a larger constant term in the bound which depends upon this regularity. A challenging situation is when the original data have a sparse representation, since this implies a very large projection dimension is required: We show how this situation can be improved for sparse binary data by applying an efficient `densifying' preprocessing, which neither changes the Euclidean geometry of the data nor requires an explicit matrix-matrix multiplication. We corroborate our theoretical findings with experiments on both dense and sparse high-dimensional datasets from several application domains.
\end{abstract}


\section{Introduction}
Randomized dimensionality reduction techniques, such as random projection (RP) \cite{dasgupta2003elementary,indyk1998approximate} and Ho's random subspace method (RS) \cite{ho1998} are popular approaches for data compression, with many empirical studies showing the utility of both for machine learning and data mining tasks in practice \cite{skurichina2002,ho1998-2,li2009,lai2006,kuncheva2010,tao2006}. For RP a key theoretical motivation behind their use is the Johnson-Lindenstrauss lemma (JLL), the usual constructive proof of which also implies an algorithm with high-probability geometry preservation guarantees for projected data. However RP is costly to apply to large or high-dimensional datasets since it requires a matrix-matrix multiplication to implement the projection, and furthermore the projected features may be hard to interpret. On the other hand RS is a particularly appealing approach for dimensionality reduction because it involves simply selecting a subset of data feature indices randomly without replacement, and so does not require a matrix-matrix multiplication to implement the projection and it retains (a subset of) the original features. RS is therefore computationally far more efficient in practice, and more interpretable than RP, but there is little theory to explain its effectiveness. Focusing on this latter problem, here we prove data-dependent norm-preservation guarantees for data projected onto a random subset of the data features. We show that provided data have a suitably regular representation then RS approximately preserves their Euclidean geometry with high probability. Our results put RS on the similar firm theoretical foundations to RP, at least for regular data, but this still leaves open the problem of non-trivial guarantees for geometry preservation when RS, i.e. a small subset of features, from data with a \emph{sparse} representation is randomly selected: Indeed as we discuss later high-probability geometry preservation guarantees for RS projection of sparse data are, in general, impossible. However, by observing that it is the regularity of the \emph{representation} of the data that is crucial to our guarantees, we see that it should be possible to obtain JLL-type guarantees by a careful preprocessing of sparse data that makes them dense in the coordinate basis (`densifying'). Unfortunately typical densifying preprocessing schemes, such as `whitening' (centering and normalizing) or Hadamard transformation\cite{ailon2009}, also require a matrix-matrix multiplication -- potentially removing the key computational advantage of RS as a dimensionality reduction approach. We attempt to address this latter problem, for the important case of sparse binary data, by applying an efficient densifying preprocessing which avoids explicit matrix-matrix multiplication by using a Householder reflection that can be carried out in linear time. Such a reflection does not alter the Euclidean geometry of the original dataset, but it makes the representation more regular so that our JLL-type guarantees can now hold non-trivially for (moderately) sparse data preprocessed in this way. We discuss the theoretical time-complexity of our approaches, and we corroborate our theoretical findings with experiments on dense and sparse high-dimensional datasets from several application domains.\enlargethispage{\baselineskip}%
 \section{Background and Motivation}
 We begin by briefly reviewing the similarities and differences between Random
 Subspace and Random Projection, 
 and we also review some key theory motivating the use of Random Projection (namely the JLL and the structure of its proof). These motivate the problem at hand and also suggest its solution.\\
In all of the following we assume, without loss of generality, that we possess a (fixed) set of $N$, $d$-dimensional real-valued vector observations to be projected, $\T_N:= \{X_{i} \in \R^d\}_{i=1}^{N}$ and we choose an integer, $k$, where $ k \in \{1,2,\ldots,d\}$ as the projection dimension. 
 \subsection{Random Subspace Projection (RS)} \label{subsec:RS}
 Random subspace is a randomized dimensionality reduction method that
 projects a data point $x \in \R^d$ onto the subspace spanned by $k$
 canonical basis vectors $e_j = (e_{j1}, e_{j2}, \ldots, e_{jd})^T$ where $e_{ji} = 1$ if $i=j$ and zero otherwise. The RS basis is chosen uniformly at random from all
 $\binom{d}{k}$ possible such subspaces of dimension $k$. In
 implementation for a single RS one simply selects a subset of $k$
 feature indices without replacement, uniformly at random from all such
 subsets of size $k$, and then discards the values of the remaining $d-k$
 features with the same $k$ feature indices being used for each data point in
 a set of observations. Thus the selected indices comprise a simple random sample without replacement of size $k$ from a population of $d$ features -- a fact that we will use later.\\
 This method was first introduced by Ho \cite{ho1998}, where an
 ensemble of decision trees employing several sets of RS projected
 data was used for a classification problem. While RS as an
 ensemble method has shown good results with many learning algorithms
 such as support vector machines, \cite{tao2006}, linear classifiers
 \cite{skurichina2002}, $k$-nearest neighbour \cite{ho1998-2} and also on a
 variety of data sets from different problem domains e.g.
 \cite{kuncheva2010,li2009,lai2006,durrant2014ACML} there is little
 explanatory theory for the success of this approach -- in particular
 there is no theory, to the best of our knowledge, for a \emph{single}
 RS.\\
 On the other hand a key advantage for RS is its very low time complexity compared to RP, namely
 $O(d)$ or $O(d\log d)$ typically to generate a subset of indices to be sampled, and $O(N)$ to construct the projected dataset. We note also that scalable parallel approaches for sampling from very large and streaming datasets have recently been devised \cite{meng2013scalable}.
 
 \subsection{Random Projection (RP)} \label{subsec:RP}
 Random projection is also a randomized dimensionality reduction method
 that projects a data point $x \in \R^d$ onto a $k$-dimensional subspace
 but this time the subspace is typically either chosen uniformly at random from all possible such subspaces of dimension $k$ in $\R^d$, or is the span of $k$ vertices of a centred hypercube chosen uniformly at random with replacement from all $2^d$ such vertices. In implementation for a single RP one generates a $k \times d$ matrix of values sampled from such a zero-mean symmetric sub-Gaussian
 distribution, and then left multiplies the data point with this RP
 matrix, the same RP matrix being used for each data point in a training
 set of observations.\\
 The RP method has its roots in geometric functional analysis, and
 entered the Machine Learning and KDD communities via Theoretical Computer Science,
 in particular seminal papers by Indyk and Motwani \cite{indyk1998approximate} and
 Arriaga and Vempala \cite{arriaga1999algorithmic}. Like RS, RP has found many
 successful applications \cite{bingham2001,venkatasubramanian2011johnson} but unlike RS the
 theoretical foundations of RP are by now quite well understood \cite{dasgupta2003elementary,matousek2008,indyk2001algorithmic}.\\
  A key theoretical result regarding RP, widely-used in theoretical
 analyses and also as heuristic justification for the application of RP,
 is the following Johnson-Lindenstrauss Lemma (JLL):
 \begin{proposition}[Johnson and Lindenstrauss, 1984]
 \label{prop_JLL}
 Let $\epsilon \in (0,1)$. Let $N,k \in \mathbb{N}$ such that $k \geq C
 \epsilon^{-2}\log N$, for a large enough absolute constant $C$. Let $V
 \subseteq \R^{d}$ be a set of $N$ points. Then there exists a linear
 mapping $R:\R^{d} \rightarrow \R^{k}$, such that for all $u,v \in V$:
 \[
 (1- \epsilon)\|u - v\|_2^{2} \leq \|Ru - Rv\|_2^{2} \leq (1+ \epsilon)\|u - v\|_2^{2}
 \]
 \end{proposition}
 The proof is constructive, and shows that an RP matrix satisfies the
 prescription for $R$ in the above theorem with positive probability. The usual approach for proving JLL is to show that, except with a small probability, an arbitrary randomly-projected vector has squared norm close to its expected value (w.r.t the draws of RP matrices); one then has the following `distributional JLL':
 \begin{proposition}
 Let $\epsilon \in (0,1)$. Let $k \in \mathbb{N}$ such that $k \geq C
 \epsilon^{-2}\log \delta^{-1}$, for a large enough absolute constant
 $C$. Then there exists a random linear mapping $R:\R^{d} \rightarrow
 \R^{k}$, such that for any unit vector $x \in \R^{d}$:
 \[
 \pr \left \{(1- \epsilon) \leq \|Rx\|_2^{2} \leq (1+ \epsilon) \right \}
 \geq 1-\delta
 \]
 \end{proposition}
 Setting $x = (u-v)/\|u-v\|$ in the above and applying union bound over all $\binom{N}{2}$ pairwise distances in a set of $N$ points obtains the
 original JLL.  JLL has been extensively studied and surveyed \cite{matousek2008}.
Unfortunately RP is in general computationally much more expensive than RS. The time complexity to generate the projection matrix is $O(kd)$, and to extract the projected data from the full data requires a matrix-matrix multiplication which is
 $O(kdN)$ in general and, although there are several approaches that consider
 increasing the sparsity of the projection \cite{achlioptas2001database,ailon2009,kane2014} to improve the hidden constants in the matrix multiplication, in practice this is still costly for large or very high-dimensional datasets. For RP matrices with $\pm 1$ entries Ailon and Liberty give an $O(Nd \log k)$ algorithm provided $k < \sqrt{d}$ \cite{ailon2009fast}. For these and similar matrices such as those in \cite{achlioptas2001database} one can also use Liberty and Zucker's Mailman Algorithm \cite{liberty2009} which, for a one-off preprocessing cost of $O(kd)$, speeds up the matrix-matrix multiplication by a factor of $O(\log d)$, though our experience is that this approach is not as fast in practice as RS and, in particular, it is very memory hungry and the data projection is slower. Finally Ailon and Chazelle \cite{ailon2009} give an $O(d\log d + N(d\log k + k^2))$ algorithm using a randomized Hadamard transformation to precondition the data so that, with high probability, it is regular in a similar sense to the one we use later in our theorems here: Indeed, we are not the first to observe that for geometry-preserving sparse projection the representation of the data is important.\\
\subsection{Motivation}
We have seen that intuitively, because of the JLL, RP distorts the Euclidean geometry of the
 original data somewhat, but with high probability (over draws of the random matrix $R$) not too much,
 while at the same time allowing one to work with a much compressed
 representation of the original data. Thus RP can yield, with the same probability, approximate solutions with performance guarantees for \emph{any} algorithm whose output depends only on the Euclidean geometry of a set of observations. For example linear classification and regression algorithms, clustering algorithms such as $k$-means, and even non-linear classifiers such as $k$-Nearest Neighbours all fit this bill. However for large or very high-dimensional data the matrix-matrix multiplication involved in the RP preprocessing is costly and may erode the benefits of working with compressed data. Moreover, as far as we are
 aware, the only known constructions for $R$ satisfying the JLL comprise sampling the entries from symmetric zero-mean sub-Gaussian distributions and, in particular, there is no known JLL guarantee for RS. Our aim here is to obtain JLL-type guarantees for RS, thus improving our understanding of this approach and at the same time providing a further route to simple, efficient, approximation algorithms with performance guarantees for a wider range of applications.
\section{Theory}
\label{sec_theorems}
Our main theoretical results are the following two theorems showing that an RS projection implies a data-dependent JLL-type guarantee. The strength of this guarantee depends on how regular the representation in which we are working is, where regularity is measured by (an upper bound on) the squared population coefficient of variation if we consider the elements of a vector as a finite population of size $d$. Our first theorem -- the `basic bound' -- is a simple Chernoff-Hoeffding type bound, while our second theorem is a tighter Serfling bound. Our second bound is much tighter than the basic bound when $k = O(d)$, but it gives a similar guarantee to the basic bound when $k \ll d$. 
The proofs are elementary and use standard tools -- we defer them to the Appendix.
For notational and analytical convenience we will write a particular RS projection in the form of a matrix $P$, where $P$ is a $d \times d$ diagonal matrix with all entries zero except for $k$ diagonal entries set to 1 with their indices chosen by simple random sampling without replacement from $\{1,2, \ldots ,d\}$. Note that left multiplying a $d \times N$ data matrix with $P$ is mathematically equivalent to RS -- viewed as a projection of the original data to a subspace of dimension $k$ embedded in $\R^d$ -- although in practice it is not how RS is usually implemented. For convenience we also define $X_i^2 := (X_{i1}^2, X_{i2}^2, \ldots ,X_{id}^2)^T$ the vector with its entries the squared components of $X_i$.

\begin{theorem}[Basic Bound]
	\label{thm_dense}
Let $\T_N:= \{X_{i} \in \R^d\}_{i=1}^{N}$ be a set of $N$ points in $\R^d$ satisfying, $\forall i \in \{1,2, \ldots , N\}$, $\|X_i^2\|_{\infty} \leq \frac{c}{d} \|X_i\|^2_2$ where $c \in \R_+$ is a constant $1 \leq c \leq d$. Let $\epsilon,\delta \in (0,1]$, and let $k \geq \frac{c^2}{2\epsilon^2}\ln{\frac{N^2}{\delta}}$ be an integer. Let $P$ be a random subspace projection from $\mathbb{R}^d \mapsto \mathbb{R}^k$. Then with probability at least $1-\delta$ over the random draws of $P$ we have, for every $i,j \in \{1,2, \ldots , N\}$:
\[
(1-\epsilon)\|X_i - X_j\|_2^2 \leq \frac{d}{k}\|PX_i - PX_j\|_2^2 \leq (1+ \epsilon)\|X_i - X_j\|_2^2
\]
\end{theorem}

\begin{theorem}[Without Replacement Bound]
	\label{thm:fpb}
Let $\T_N:= \{X_{i} \in \R^d\}_{i=1}^{N}$ be a set of $N$ points in $\R^d$ satisfying, $\forall i \in \{1,2, \ldots , N\}$, $\|X_i^2\|_{\infty} \leq \frac{c}{d} \|X_i\|^2_2$ where  $c \in \R_+$ is a constant $1 \leq c \leq d$. Let $\epsilon,\delta, f_k \in (0,1]$, where $f_k := (k-1)/d$ and let $k$ such that $k/(1-f_k) \geq \frac{c^2}{2\epsilon^2}\ln{\frac{N^2}{\delta}}$ be an integer. Let $P$ be a random subspace projection from $\mathbb{R}^d \mapsto \mathbb{R}^k$. Then with probability at least $1-\delta$ over the random draws of $P$ we have, for every $i,j \in \{1,2, \ldots , N\}$:
\[
(1-\epsilon)\|X_i - X_j\|_2^2 \leq \frac{d}{k}\|P(X_i - X_j)\|_2^2 \leq (1+ \epsilon)\|X_i - X_j\|_2^2
\]
\end{theorem}
Furthermore we also have:
\begin{corollary}[to either bound]
\label{cor:tbb}
Under the conditions of Theorem \ref{thm_dense} or \ref{thm:fpb} respectively, for any $\epsilon,\delta \in (0,1]$, with probability at least $1-2\delta$ over the random draws of $P$ we have:
{\small
\[
\left(X_i^T X_j - \epsilon\|X_i\|\|X_j\|\right) \leq \frac{d}{k}(PX_i)^T(PX_j)\leq \left(X_i^T X_j + \epsilon\|X_i\|\|X_j\|\right)
\]
}
\end{corollary}
\emph{Comment on Corollary \ref{cor:tbb}} For RP matrices with zero-mean sub-Gaussian entries, a $1-\delta$ guarantee for projected dot products is proved in \cite{kaban2015improved}. The proof technique used there is not directly tranferrable to RS, although we speculate that for small enough $c$ it could be adapted to RS using some results of Matousek in \cite{matousek2008}.
\subsection{Discussion of Bounds}
These theorems and their corollaries show that we have high probability guarantees on Euclidean geometry preservation for sufficiently regular datasets for RS and provided the dimension of the projected subspace, $k$, is chosen large enough. We note that up to constant terms this is the same guarantee as we have for the existing JLL for RP, therefore it is of optimal order for any linear dimensionality reduction scheme\cite{larsen2014johnson}, but for a fixed $k$ the RS projection is typically orders of magnitude faster than RP. However, there is a trade-off involved since if $c$ is large the projection dimension required will generally be greater than for RP, indeed for RP our $c^2$ can be replaced by a 
\enlargethispage{\baselineskip}single-digit constant (either 2 or 8) which only depends on the choice of a Gaussian or sub-Gaussian RP matrix $R$ and not on the data.\\
Our bounds hold for an RS projection of any set of data vectors meeting the given conditions which may seem rather surprising: For example, if we consider a binary vector $X$ with only one non-zero component then it is straightforward to check that under RS with probability $1-(d-k)/d$ the projected vector is the zero vector, otherwise it has norm 1, and in neither case is the squared norm of $PX$ close to its expected value $\frac{k}{d}\|X\|_2^2$ in general. Furthermore it is easy to verify for any vector with $s < d$ non-zeros that the number of non-zero components sampled by RS has a Hypergeometric$(s,d,k)$ distribution and so if $s \ll d$ this problem remains and the norms of most projections will be very far from their expected value. However, we note that in such cases the regularity constant $c \in [1,d]$ will also be close to $d$ and thus there will only be a non-zero probability guarantee of norm preservation for $k = d$ when, of course, the guarantee holds trivially. Thus for RS it is not possible to avoid some regularity condition on the data and to also have non-trivial JLL-type norm-preservation guarantees, and for fixed $\epsilon$ the projection dimension $k$ must generally be larger than it would be for RP but this is the price to pay for using RS projection. On the other hand we see from our theorem that it is not sparsity of the data \emph{per se} that causes a problem, rather \emph{it is sparsity in the data representation}. We will leverage this fact later in Section \ref{sec:sparse} when we present an efficient way to finesse this problem for the special, but common, case of sparse binary data.

\section{Empirical Corroboration of Theory}
We now present empirical results which corroborate our theorems. We demonstrate that we have norm preservation for Random Subspace projections (RS) as proven by our theory, and we compare RS projection with two RP variants as well as to principal components analysis (PCA) to see that in practice -- given a suitable choice of $k$ -- RS works as well as these alternative solutions. 

\subsection{Datasets}
We use two datasets, the first is a collection of natural images \cite{usc-sipi} similar to those used by Bingham and Manila in \cite{bingham2001}; and the second is the Dorothea dataset from the 2003 NIPS feature selection challenge, which is a very sparse and very high dimensional binary drug-discovery dataset split into three for purposes of the NIPS competition. The characteristics of the datasets are summarized in Tables \ref{DenseDataset} and \ref{SparseDataset}. 
\begin{table}[]
\centering
\begin{tabular}{|l|l|l|l|}
\hline
Name       & Description       & Image Size & $c$\\ \hline
5.1.09     & Moon Surface      & 256x256  & 3.50 \\ \hline
5.1.10     & Aerial            & 256x256   & 2.44 \\ \hline
5.1.11     & Airplane          & 256x256  & 7.92  \\ \hline
5.1.12     & Clock             & 256x256    & 5.03\\ \hline
5.1.14     & Chemical plant    & 256x256 & 2.92    \\ \hline
5.2.08     & Couple            & 512x512   & 2.64 \\ \hline
5.2.09     & Aerial            & 512x512    & 4.10\\ \hline
5.2.10     & Stream and bridge & 512x512  & 2.34  \\ \hline
5.3.01     & Man               & 1024x1024  & 2.23 \\ \hline
5.3.02     & Airport          & 1024x1024 & 3.82 \\ \hline
boat.512   & Fishing Boat      & 512x512  & 2.89  \\ \hline
7.1.01     & Truck             & 512x512    & 3.02\\ \hline
7.1.02     & Airplane          & 512x512  & 9.69  \\ \hline
7.1.03     & Tank              & 512x512    & 2.89 \\ \hline
7.1.04     & Car and APCs      & 512x512 & 2.72   \\ \hline
7.1.05     & Truck and APCs    & 512x512  & 2.37  \\ \hline
7.1.06     & Truck and APCs    & 512x512  & 2.37  \\ \hline
7.1.07     & Tank              & 512x512  & 2.89  \\ \hline
7.1.08     & APC               & 512x512   & 4.85 \\ \hline
7.1.09     & Tank              & 512x512   & 2.40 \\ \hline
7.1.10     & Car and APCs      & 512x512  & 2.73  \\ \hline
7.2.01     & Airplane (U-2)    & 1024x1024 & 4.12 \\ \hline
elaine.512 & Girl (Elaine)     & 512x512   & 2.25 \\ \hline
\end{tabular}
\caption{Natural Image Dataset: $c$ is the regularity constant in the bounds which here was calculated from each complete image.}
\label{DenseDataset}
\end{table}
\begin{table}
\centering
\begin{tabular}{|c|c|c|c|c|c|}
\hline
Name           & \begin{tabular}[c]{@{}l@{}}Number\\ of \\observations\end{tabular} &  \begin{tabular}[c]{@{}l@{}}Features\\with non-\\ zero\\ variance($d$)\end{tabular}&  $c$ & $c' $      \\ \hline
.test  & 800                    & 91362           &  71.21  & 67.27              \\ \hline
.train & 800                    & 88119           &  68.63 &  64.69  \\ \hline
.valid & 350                    & 72113              & 55.94 & 52.02  \\ \hline
\end{tabular}
\caption{Dorothea Dataset: $c$ is the regularity constant in the bounds which here was calculated from each dataset split. $c'$ is the corresponding constant for the Householder-transformed data.}
\label{SparseDataset}
\end{table}
\begin{table}
\centering
\begin{tabular}{|l|c|}
\hline
Method                       & Norm Scaling Factor                    \\ \hline
Gaussian Random Projection   & $\sqrt{\frac{d}{k}}$                 \\ \hline
Sparse Random Projection     & $\sqrt{\frac{1}{k}}$                 \\ \hline
Random Subspace  &  $\sqrt{\frac{d}{k}}$ \\ \hline
Principal Component Analysis & $\sqrt{\frac{Trace(\Sigma)}{Trace(\Sigma(1:k))}}$ \\ \hline
\end{tabular}
\caption{Theoretical norm-scaling quantities for the various projection schemes.}
\label{tableMethod}
\end{table}

\subsection{Experimental Procedure}
For the image data we used all twenty-three publicly available natural grayscale images from the USC-SIPI natural image dataset and we omitted the synthetic images; a short description and the sizes of the images is given in Table 1.  We follow the same protocol as Bingham and Manilla \cite{bingham2001}; for each of the images, we select the top-left corner of a 50x50 pixel window in each image uniformly at random and reshape to a vector with 2500 dimensions, repeating this one thousand times for each of the images.  We then project the vectors using RS, othornormalized Gaussian random projection (RP), Achlioptas sparse random projections (SRP) (with $P_{i,j}= \pm 1$ with probability $\frac{1}{6}$, $0$ with probability $\frac{2}{3}$) and also the first $k$ eigenvectors from applying PCA to the full sample of the one thousand vectors. The projected vectors were all scaled according to the values in Table \ref{tableMethod}. Note that a scaling correction for PCA was not employed in \cite{bingham2001} where it was claimed a straightforward rule is difficult to give. In fact one can verify the average scaling for PCA projected vectors (over the dataset) in the squared Euclidean norm should be $\frac{Trace(\Sigma)}{Trace(\Sigma(1:k))}$ and so we use the square root of this. We let the projection dimension $k$ range from $5$ to $600$ in increments of $5$.\\
For the image data this procedure was repeated for all twenty three image files for each projection approach.
For Dorothea for each of the three dataset splits, we first removed features with zero variance from the data set (these were all zero-valued features) but to avoid possible confounds we carried out no other filtering.  We then projected the data using RP, SRP and RS as before with the projection dimension $k\in \{5,\ldots, 70,000\}$ for RS, $k\in \{5,\ldots, 2,750\}$ for RP and SRP. We also applied RS to the data transformed by a fixed Householder reflection since this reduces the value of $c$ in an efficient manner for these data -- details follow shortly. Since the runtime and memory overhead is prohibitive, we did not run PCA on this dataset.\\

For both types of data we randomly selected one hundred observations and for each possible pair of these we calculated the $\ell_2$ norm of the difference between the (scaled) projected observations $\|P(u-v)\|$ and the original points $\|u-v\|$. We then calculated the ratio between the (scaled) projected norm and the true norm  $\frac{\|P(u-v)\|}{\|u-v\|}$ for each observation where the scaling constants used were those in Table \ref{tableMethod}.\\
For the image data we plot, for each choice of $k$, the average of this value over all images as well as the $5$-th and $95$-th percentiles for the different ratios in Figure \ref{figEintD}. We also plot the runtime, for the image data, for each projection method versus $k$ in Figure \ref{figRunTimeD}.\\
For Dorothea we repeated our experiments five times on each dataset split, to obtain an average over fifteen runs. We report the mean ratio of the norms $\frac{\|P(u-v)\|}{\|u-v\|}$ as well as the $5$-th and $95$-th percentiles in Figure \ref{figEIntS}. The average runtime for each different approach can be seen in Figure \ref{figRunTimeS}.

\subsection{Smoothing binary data using a fixed Householder transform}\label{sec:sparse}
A Householder transform $H$ is given by $H:= I - 2vv^T$ where $I$ is the identity matrix and $\|v\|_2 = 1$. One can easily check that $v$ is an eigenvector of $H$ with eigenvalue $-1$, all other eigenvalues are $1$, and that $H = H^T = H^{-1}$. Geometrically, $H$ is therefore a reflection about a hyperplane through the origin with normal vector $v$ and, in particular, $\ell_2$ norms are preserved by $H$: $\|HX\|_2 = \|X\|_2$ for any $X$. Moreover $HX = X - 2v(v^{T}X)$ so one need not evaluate the matrix multiplication explicitly. The benefit of $H$ to us in the setting of RS projection is that it provides an efficient way in which to `densify' sparse binary data. As already discussed, if X is a sparse binary vector then $c \simeq d$ and we have no non-trivial geometry-preservation guarantees for a RS projection. Indeed if X is binary with $s \ll d$ non-zeros entries then $c=\frac{d}{s}$. However, we see from Theorem 3.1, that if we can reduce $c$ then a non-trivial guarantee is possible. Thus an efficient method to reduce $c$ would be useful for these data. We have the following theorem:

\begin{theorem}[Densification]\label{thm:dense}
Let $X \in \{0,1\}^d$ with $s$ non-zeros.  
Let $v \in \R^d $, $v_j  = \frac{1}{\sqrt{d}} , \forall j \in \{ 1,2 \dots d\} $, and let $H := I - 2vv^T$ where $I$ is the identity matrix.
Denote by $X^2$ and $(HX)^2$, the vectors consisting of the squared entries of $X$ and $HX$ respectively.\\
Let $c =  \frac{d \|X^2 \|_{\infty}}{\|X\|^2_{2}} = \frac{d}{s}$  and let $c' =  \frac{d \|(HX)^2 \|_{\infty}}{\|HX\|^2_{2}}$.\\
Then if $s< \frac{d}{2}$, $c' < c$,\\
if $s=\frac{d}{2}, c'= c$, and\\
if $s> \frac{d}{2}$, $c' > c$. 
\end{theorem}

\begin{proof}
Since $H$ is a reflection and $X$ is binary with $s$ non-zeros, we have $\|HX\|^2_2 = \|X\|^2_2 = s$ and $\|X^2 \|_{\infty} = 1$.
Thus to compare $c$ and $c'$,  we only need to consider what values $\|(HX)^2 \|_{\infty}$ can take.\\
Now, the $j$-th entry of $HX$ is:
\begin{eqnarray}
(HX)_{j} = X_{j} - 2v_j v^{T}X =  X_j - 2v_j \sum_{i=1}^d \frac{X_{i}}{\sqrt{d}} \nonumber \\
= X_{j} - 2\frac{1}{\sqrt{d}}\sum_{i=1}^d \frac{X_{i}}{\sqrt{d}} = X_{j} - 2 \frac{s}{d}
\label{HX}
\end{eqnarray}
So when $X_{j}=1$, $(HX)_{j}= \frac{d-2s}{d}$ and when $X_{j}=0$, $(HX)_{j}= -\frac{2s}{d}$.\\
Next $\|(HX)^2\|_{\infty} = \max_{j \in \{1,2,\ldots,d\}} |(HX)_j^2|$ so, checking cases:\\
For $s < \frac{d}{4}, \left(\frac{d-2s}{d} \right)^2 = \frac{d^2-4sd+4s^2}{d^2}  >  \frac{4s^2}{d^2} =\left( \frac{-2s}{d}\right) ^2 $ So, $\| (HX)^2\|_\infty = \left(\frac{d-2s}{d}\right)^2$ and $c' =\frac{d \left(\frac{d^2 -2sd+4s^2}{d^2}\right)}{s}= \frac{d^2 - 4sd+ 4s^2}{sd} = \frac{d}{s} - 4 + \frac{4s}{d} < \frac{d}{s} = c$\\
For $\frac{d}{4} \leq s < \frac{d}{2}, \frac{d^2-4sd+4s^2}{d^2}   \leq  \frac{4s^2}{d^2}$. So,$\|(HX)^2\|_\infty = \left(\frac{-2s}{d} \right)^2$ and $c' =\frac{d \left(\frac{4s^2}{d^2}\right)}{s}= \frac{4s}{d}  <  \frac{d}{s} = c$\\
Finally, for $\frac{d}{2} \leq s \leq d, \frac{d^2-4sd+4s^2}{d^2} \leq  \frac{4s^2}{d^2}$ and $c' = \frac{4s}{d} \geq \frac{d}{s} = c$.\\
This completes the proof.
\end{proof}
\emph{Comments on Theorem \ref{thm:dense}} Picking the theorem apart we see that when $s<\frac{d}{4}$, our Householder transform ensures $c-c'\geq 3$ which translates to (at least) a 9-fold reduction in $k$, however note that $k$ will still typically remain large compared to the corresponding quantity for RP since $c = d/s$ will usually have been large in the first place, on the other hand when $s=\frac{d}{4}$ all transformed entries have the same absolute value and $c'=1$ is minimal and we have a stronger guarantee than for RP! When $\frac{d}{4} < s <\frac{d}{2}$ our Householder transform improves $c'$ by a factor of $\frac{d^2}{4s^2}$, and when $s$ is any greater than $d/2$ then applying this Householder transform instead makes the data \emph{less} regular. For very sparse, very high-dimensional, binary data the improvement from our approach is therefore moderate -- we present experiments on such data later -- but for moderate values of $d$ or of $s$, for example if $d$ is just a few thousand or if $s \simeq d/4$, then the improvement from using RS with our choice of $H$ can be large.\\

Note that this Householder transform can be applied in linear (in $d$) time, and we only need knowledge of the number of non-zero entries in the vector to apply it. If we have stored the non-zero indices for each of the $X_i$ then the transform can be done in time linear in $s$. In any case, avoiding explicit pre-multiplication of the data by $H$ allows us to obtain smaller distortion in norms for a fixed value of $k$, or the same error for a smaller value of $k$, with no increase in the theoretical time complexity of RS and at very low computational cost in practice -- see Figure \ref{figRunTimeS}.

\begin{figure}[ht!]
\centering
\includegraphics[width=0.5\textwidth]{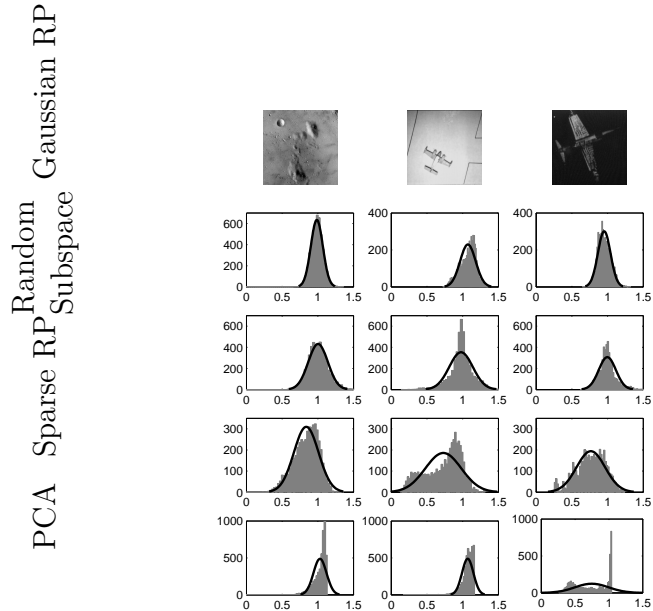}
  \put(-250,42){\rotatebox{90}{PCA}}
  \put(-250,77){\rotatebox{90}{Sparse RP}}
  \put(-260,127){\rotatebox{90}{
  \begin{tabular}[c]{@{}l@{}}Random \\Subspace\end{tabular}
  }}
  \put(-250,182){\rotatebox{90}{Gaussian RP}}
\caption{Fixed $k$, small $c$: Histograms of $\frac{\|P(X_i -X_j)\|}{\|X_i -X_j\|}$ for $k=50$ dimensions on three representative images with overlaid normal density plots, $n=4950$. }
\label{figHistD}
\end{figure}

\begin{figure*}
\includegraphics[width=\textwidth]{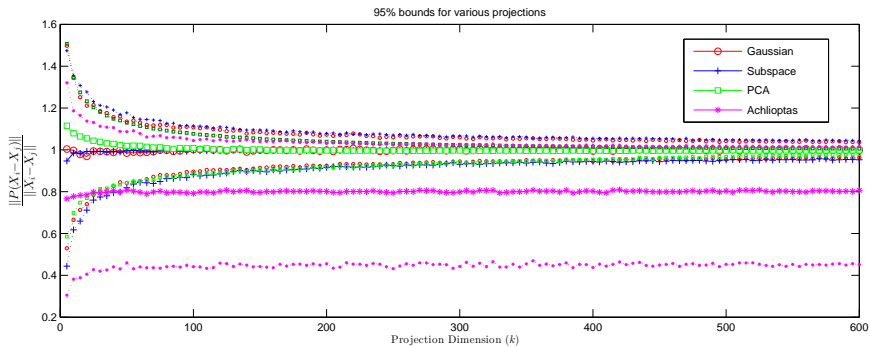}
\caption{Mean and 5th and 95th percentiles of $\frac{\|P(X_i -X_j)\|}{\|X_i -X_j\|}$ for image data vs. $k$. We see that for $k \gtrsim 80$ Gaussian RP and RS are indistinguishable on these data. Note also the 5th percentile for SRP cf. Figure \ref{figHistD}: Sparse RP frequently seems to underestimate norms.}
\label{figEintD}
\end{figure*}

\begin{figure}[ht!]
\includegraphics[width=0.5\textwidth]{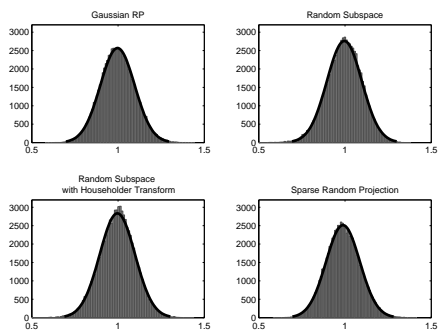}
\centering
\caption{Fixed $\epsilon$, large $c$: Histograms of $\frac{\|P(X_i -X_j)\|}{\|X_i -X_j\|}$ for Dorothea dataset with $k_{rp}=50$ (RP and SRP, top left and bottom right plots) and comparison with RS and Householder + RS when $k_{rs} = c^2 \times k_{rp}$ dimensions with overlaid normal density plots, $n=4,950$. We see that errors behave nearly identically for RP and RS as predicted by theory.}
\label{figHistS}
\end{figure}

\begin{figure*}[ht!]
\centering
\includegraphics[width=\textwidth]{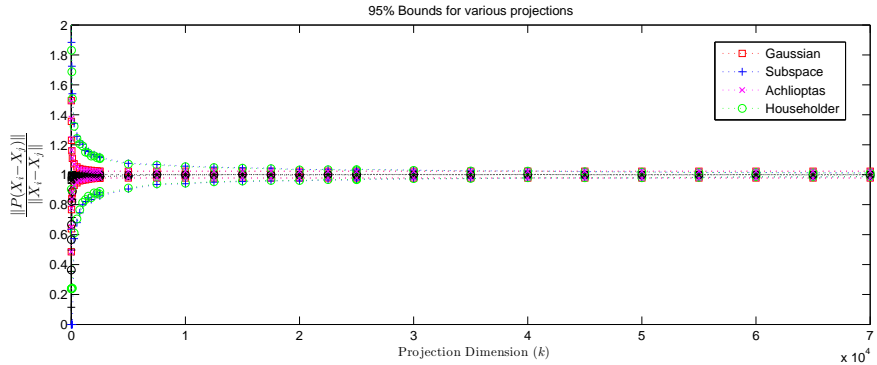}
\caption{Mean and 5th and 95th percentiles of $\frac{\|P(X_i -X_j)\|}{\|X_i -X_j\|}$ for Dorothea vs. $k$. We see that for RS a much higher $k$ is required than for RP, though RS eventually catches up. The main improvement from using Householder transform is found for lower values of $k$, but RP is still better on these data.}
\label{figEIntS}
\end{figure*}

\begin{figure}[h!t]
\centering
\includegraphics[width=0.5\textwidth]{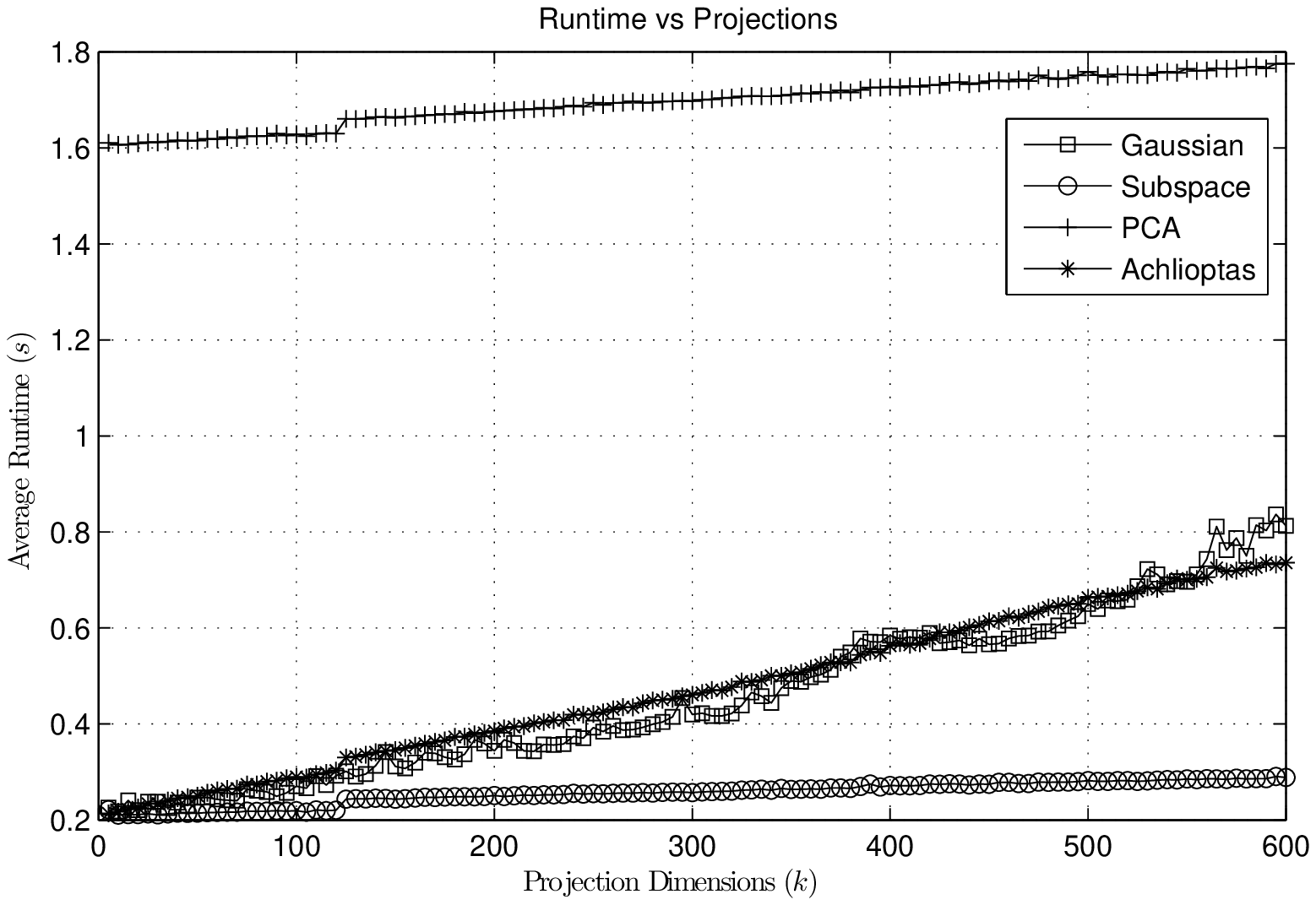}
\caption{Comparison of the runtime on dense image datasets with dimensionality $d=2500$ }
\label{figRunTimeD}
\centering
\includegraphics[width=0.5\textwidth]{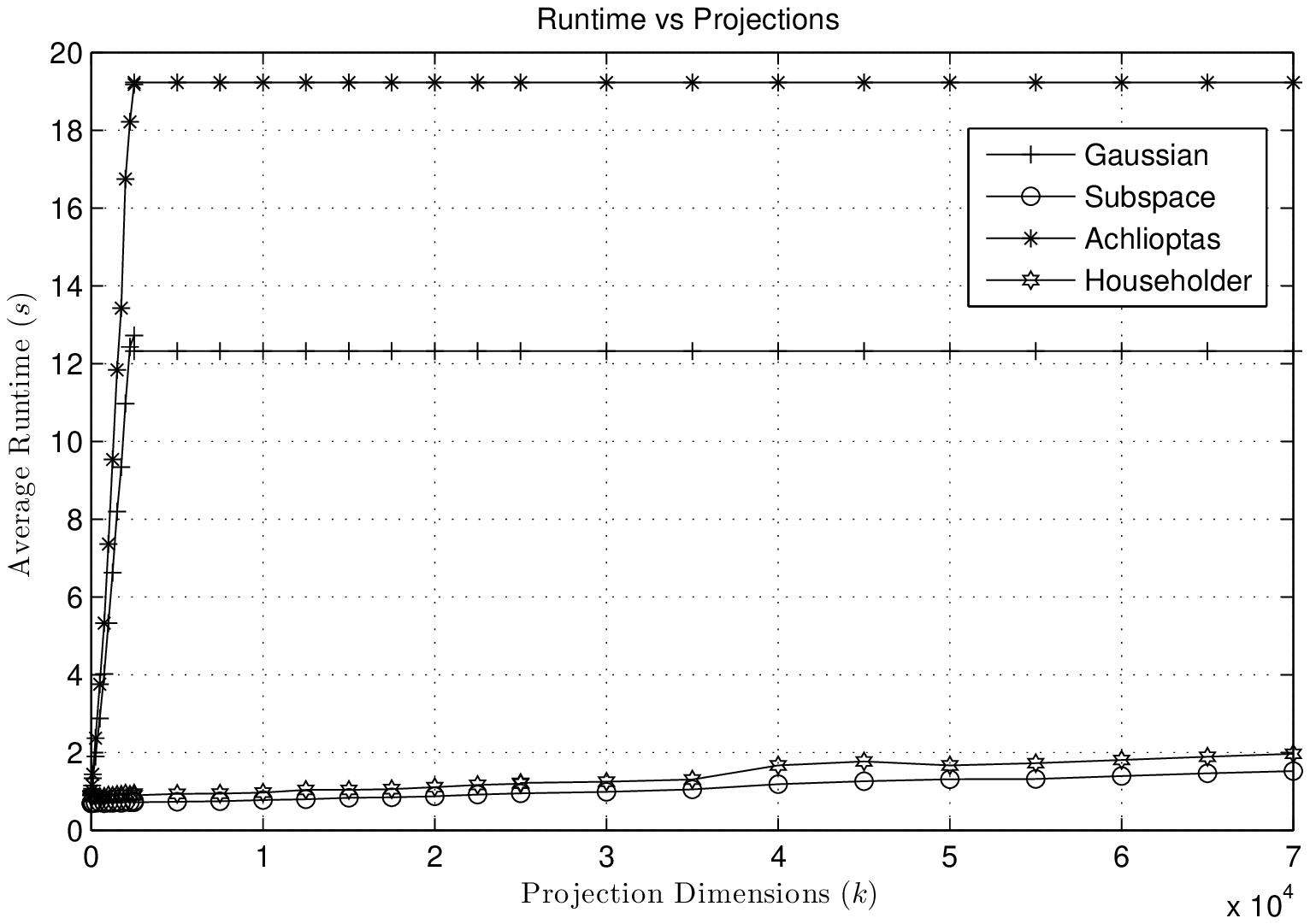}
\caption{Comparison of the runtime on Dorothea with $d \simeq 100,000$, $s/d \simeq 0.1$. Note that due to memory limitations the runs for RP and SRP are from $k = 5$ to $2750$ and the Gaussian RP was not orthonormalized. We plot the runtime for $k = 2750$ for values of $k > 2750$. Gaussian RP was faster than SRP here because generating the SRP matrix was slow for such large values of $d$.}
\label{figRunTimeS}
\end{figure}

\subsection{Experimental Results and Discussion}
\subsubsection{Random Subspace on Image Data}
Our experimental results corroborate our theory, and we observe for natural image data that RS indeed gives similar performance in terms of norm preservation to RP and, surprisingly, better performance than SRP on these data (as does RP) -- see Figures \ref{figHistD} and \ref{figEintD}. Given the small values of $c$ estimated for these data (See Table \ref{DenseDataset}) the similar performance to RP is broadly in line with what we would predict from theory, indeed Figure \ref{figEintD} shows that RS is nearly indistinguishable from the computationally more expensive RP on these data. On the other hand one remarkable finding is that the distribution of norms for SRP is left-skewed here, and there is ample evidence that SRP consistently tends to underestimate distances between points when the correct theoretical scaling is applied, at least on these data. In this respect SRP does worst on images such as the high contrast one above the centre column of Table \ref{figHistD}, where we might instead reasonably expect RS to suffer from such a problem: Indeed, the normal fit for RS applied to this image does show heavier tails for RS than for RP, but unlike SRP the error distribution is symmetric and the centre of mass is in the right place at 1. We don't have a reasonable explanation for why SRP should be worse than RS on these images, but as we see clearly in Figure \ref{figEintD} this problem persists even as $k$ grows. A further interesting finding is that, unlike the results reported in \cite{bingham2001}, the performance of PCA scaled according to the scheme outlined in Table \ref{tableMethod} is -- for a large enough choice of $k$ -- superior to all three random alternatives we considered. This is to be expected since PCA maximises the retained within-feature variance on the projected sample and the scaling we proposed is adaptive in a non-linear way to this quantity, unlike the other alternatives which do not consider local properties of the data cloud and use a scaling that is linear in $k$. How far similar outcomes would hold for other types of data remains for future research, but we note that it must depend on both the choice of $k$ and also on the rate at which the spectrum of the sample covariance matrix -- the eigendecomposition of which gives the principal components -- decays, since the scaling correction we apply to PCA is piecewise constant in $k$ with a non-uniform step size. We also note that, unlike for RP, SRP and RS, for PCA there is no theory to guide the user's choice of $k$ \emph{a priori} even if one has access to the constant $c$ we require in our RS bounds.\\

Finally we look at the computational cost of the different approaches considered: These are compared in Figure \ref{figRunTimeD}.
For a fixed $k$ there is of course a significant runtime improvement in using RS compared to RP and SRP. On these data it seems that choosing $k$ the same for RS, RP and SRP works equally well and so, everything else being equal, one would likely prefer RS to RP or SRP here. Note that in general however, for fixed error, the projection dimension $k$ for RS will be around $c^2$ times greater than for RP or SRP so there is a trade-off. Whether one would prefer to use RS with a larger $k$ than for RP (for the same high-probability error guarantee) will depend on problem specifics such as the time complexity of the algorithm receiving the projected data with respect to the dimension, or whether it is more important to classify or to train quickly. Finally PCA is, of course, computationally much more expensive when compared to the other three approaches, but we see that with the proper scaling term on these data it outperforms them in terms of geometry preservation. Thus for PCA there is essentially the same accuracy-vs-complexity trade-off as for RS.

\subsubsection{Random Subspace on Dorothea Dataset}
The Dorothea data is very high-dimensional with only around 10\% of entries non-zero and for these data the theory predicts that we will have poor norm preservation from RS compared to RP except when $k$ is very large, but that this situation will be improved by applying our Householder transformation.
Our experimental results -- see Figures \ref{figEIntS} and \ref{figHistS} -- show that indeed is the case. RS does catch up with RP and SRP in terms of error eventually, but both RP and SRP attain smaller error much more quickly than RS. When the Householder transform $H$ is applied to the data first, then RS catches up quicker, but still requires a higher $k$ than RP or SRP -- in this case because the densifying regularization applied by $H$ only gives a small improvement in $c$ and therefore, for such super-sparse high-dimensional data, it does not bring great improvement.
On the other hand we see in Figure \ref{figHistS} that after scaling the projected dimension required for RP by $c^2$ that RS indeed has comparable (and sometimes better) error performance than RP or SRP, and that the Householder transform slightly reduces the error variance.  We also see in Figure \ref{figEIntS} that interestingly unlike for the image data the scaled SRP does not tend to consistently underestimate norms, and all approaches (eventually) have their centre of mass at 1. Finally despite the increased projection dimension, for a fixed error guarantee either variant of RS still gives us significantly improved runtime compared to RP and SRP (See Figure \ref{figRunTimeS}). 

\section{Conclusions and Future Work}
We proved that random subspace can act as a norm preserving projection and showed how this norm preservation property depends on the regularity of the data features. We corroborated this theory empirically and saw that, for regular data such as natural images, random subspace can achieve geometry preservation performance comparable to random projection, but with a significant runtime improvement. 

We also provided an efficient technique for densifying sparse data to close the performance gap between RP and RS on sparse data, though very sparse data remains challenging for this approach.

We are currently working on improving on uniform sampling to construct the RS projection, e.g. by taking account of stratification in the features, or by constructing $H$ with regard to the distribution of features. 

Finally we note a connection with dropout regularization, which is essentially a single RS projection applied to a layer of nodes (usually the last)in a neural network. Based on earlier work \cite{durrant2010compressed, durrant2013sharp} we conjecture that the number of retained nodes in a dropout scheme should be logarithmic in the number of classes to guarantee good classification performance, and we plan to explore this in future work.

\appendix
\section{Proofs of Bounds}
We will use the following two lemmas which are from \cite{hoeffding1963probability, serfling1974probability}.
\begin{lemma}[Hoeffding, 1963 \cite{hoeffding1963probability} Theorem 2.]
\label{lemma:hoeff}
Let $X_1, X_2, \ldots , X_k$ be independent random variables such that, $\forall i \in 1,2, \ldots, k$ we have $X_i \in [a_i,b_i]$ with probability 1. Denote by $S_k := \sum_{i=1}^k X_i$ and fix $t>0$. Then:
\[
\pr \l \{ \l|S_k - \E[S_k] \r| \geq t \r \} \leq 2 \exp \l(- \frac{2t^2}{\sum_{i=1}^k (b_i - a_i)^2}\r)
\]
\end{lemma}
\begin{corollary}[to Lemma \ref{lemma:hoeff}, \cite{hoeffding1963probability} Section 6.]
\label{corr:basic}
Let $C := c_1,c_2, \ldots , c_d$ be a finite population of $d$ values where $\forall j = 1,2,\ldots,d$ we have $c_j \in [a_j,b_j]$ with probability 1. Let $X_i$ and $Y_i$, $i=1,2,\ldots,k$ be samples without and with replacement from $C$ respectively and define by $S_k(X)$ and $S_k(Y)$ the corresponding sample totals. Fix $t>0$. Then it holds that:
\[
\pr \l \{ \l|S_k(X) - \E[S_k(X)] \r| \geq t \r \} \leq \pr \l \{ \l|S_k(Y) - \E[S_k(Y)] \r| \geq t \r \}
\]
\end{corollary}
Note that $\E[S_k(X)] = \E[S_k(Y)]$, thus we may bound the probability of a large deviation in the sample total from its expectation in the case of a (non-independent) sample without replacement by the corresponding probability for an independent sample with replacement.
\begin{lemma}[Serfling, 1974 \cite{serfling1974probability} Corollary 1.1.]
\label{lemma:serf}
Let $C := c_1,c_2, \ldots , c_d$ be a finite population of $d$ values where $\forall j = 1,2,\ldots,d$ we have $c_j \in [a_j,b_j]$ with probability 1. Let $X_i$, $i=1,2,\ldots,k$ be a simple random sample without replacement from $C$. Denote by $S_k := \sum_{i=1}^k X_i$ and define the sampling fraction $f_k := (k-1)/d$. Fix $t>0$. Then:
\[
\pr \l \{ \l|S_k - \E[S_k] \r| \geq t \r \} \leq 2 \exp \l(- \frac{2t^2}{(1-f_k)\sum_{i=1}^k (b_i - a_i)^2}\r)
\]
\end{lemma}
\emph{Comment:} Since $1-f_k = (d - k + 1)/d < 1$ Lemma \ref{lemma:serf} gives a strictly tighter bound than Lemma \ref{lemma:hoeff} for sampling without replacement, but brings in a dependence on $d$. We note that bounds for sampling without replacement which are somewhat tighter than those in \cite{serfling1974probability} when $k \simeq d$ were recently proved in \cite{bardenet2015concentration}, in particular an empirical variant for when the population parameters are unknown. In our proof each population is a fixed vector of known length where the data dimension $d$ is the population size and the projection dimension $k$ is the sample size; thus in our setting we have access to both the full population and its parameters.
\subsection{Proof of Basic Bound}\label{subsec:basic}
We prove the basic bound using Lemma \ref{lemma:hoeff} and Corollary \ref{corr:basic} and our without replacement bound then follows directly. The basic idea is to treat each vector as a finite population of size $d$ and RS as a simple random sample of size $k$ without replacement from it in the above lemmas, and then  follow the line of argument in the usual proof of the JLL.\\
Let $X \in \R^d$ be an arbitrary, but fixed, real-valued vector and without loss of generality let $\|X\|_2^2 = 1$  (since otherwise we can take $X = Z/\|Z\|_2$). Denote by $X^2 := (X_1^2,X_2^2, \ldots , X_d^2)^T$ the vector containing the squared components of $X$. Assume that $\|X^2\|_{\infty} \leq \frac{c}{d}\|X\|_2^2$.\\
Now let $P \in \M_{d \times d}$ be a projection onto $k$ standard coordinate vectors, where the projection basis is chosen by sampling uniformly at random from all $\binom{d}{k}$ possible such bases. As noted already in Subsection \ref{subsec:RS} this is mathematically equivalent to an RS projection. Then in every random $P$ it holds that $k$ of the $P_{ii} =1$ and every other entry of $P$ is zero so $\tr(P) = k$ for any random $P$, and therefore $\tr(\E[P]) = \E[\tr(P)] = k$. Furthermore since $\pr\{P_{ii} = p\} = \pr\{P_{jj} = p\}$ for all $i,j \in \{1,2,\ldots,d\}$ and $p \in \{0,1\}$, it follows that $\E[P_{ii}] = \E[P_{jj}] = k/d, \forall i,j$ by symmetry. Thus $\E[P] = \frac{k}{d}I$ and $\E[\|PX\|_2^2] = \frac{k}{d}\|X\|_2^2$, where both expectations are taken with respect to the random draws of $P$ and we used the fact that $P^{T}P = PP = P, \forall P$.\\
We want to upper bound the following probability:
{\small
\[
\pr\l\{|\frac{d}{k}\|PX\|_{2}^2 - \|X\|_{2}^2| \geq \epsilon\r\} = \pr\l\{|\frac{d}{k}\|PX\|_{2}^2 - \frac{d}{k}\E\l[\|PX\|_{2}^2\r]| \geq \epsilon\r\}
\]
} 
We give details for one side of the inequality using the basic Hoeffding bound, the other cases proceed along the same lines. Now, for any fixed instance of $P$ denote by $I$ the index set such that $i \in I \iff P_{ii} = 1$. Then:
{\small\[
\pr\l\{\|PX\|_{2}^2 \geq \frac{k}{d}\epsilon + \E\l[\|PX\|_{2}^2\r]\r\} = \pr\l\{\sum_{i \in I}X_{i}^2 \geq \frac{k}{d}\l(\epsilon + \sum_{i=1}^{d}X_{i}^2\r)\r\}
\]} 
where the sample total $\sum_{i \in I}X_{i}^2$ is estimated from a sample of size $k$ without replacement. Applying Lemma \ref{lemma:hoeff} and Corollary \ref{corr:basic} we then have:
{\small
\begin{eqnarray}
&\pr\l\{\sum_{i \in I}X_{i}^2 \geq \frac{k}{d}\l(\epsilon + \sum_{i=1}^{d}X_{i}^2\r)\r\}\nonumber\\
&= \pr\l\{\frac{d}{k}\|PX\|_{2}^2 - \|X\|_{2}^2 \geq \epsilon\r\} \leq \exp\l(-\frac{2k\l(\frac{\epsilon}{d}\r)^2}{\|X^2\|_{\infty}^2}\r)\nonumber
\end{eqnarray} 
}
The lower bound proceeds similarly and yields the same probability guarantee for a single fixed vector:
{\small
\[
\pr\l\{\|X\|_{2}^2 - \frac{d}{k}\|PX\|_{2}^2 \geq \epsilon\r\} \leq \exp\l(-\frac{2k\l(\frac{\epsilon}{d}\r)^2}{\|X^2\|_{\infty}^2}\r)
\]
}
Thus by union bound, and using the condition on the theorem $\|X^2\|_{\infty} \leq \frac{c}{d}\|X\|_2^2$ to kill the unwanted dependence on $d$, we obtain the following guarantee for an arbitrary unit-norm vector $X$:
{\small
\begin{equation}
\pr\l\{\l|\|X\|_{2}^2 - \frac{d}{k}\|PX\|_{2}^2 \r| \geq \epsilon\r\} \leq 2\exp\l(-\frac{2k\epsilon^2}{c^2\|X\|_{2}^4}\r)
\label{eq:fin}
\end{equation}
}
To complete the proof we consider a set, $\T_N$, of $N$ vectors in $\R^d$ and let $X_i$ and $X_j$ be any two vectors in this set. Instantiating $X$ in \ref{eq:fin} as $(X_i - X_j)/\|X_i - X_j \|_{2}$ and then applying union bound again over all $\binom{N}{2} < N^2/2$ inter-point distances in $\T_N$ we obtain, for all pairs $X_i, X_j \in \T_N$ simultaneously, it holds that:
\[
\pr\l\{\l|\|X_i - X_j \|_{2}^2 - \frac{d}{k}\|PX_i - PX_j \|_{2}^2\r| \geq \epsilon\r\} \leq N^2 \exp\l(-\frac{2k\epsilon^2}{c^2}\r)
\] 
Where we substituted $\|X\|_{2}^4 = 1$ in RHS. Finally, setting the probability upper bound on the RHS to $\delta$ and solving for $k$ gives the theorem.

For the without replacement bound, one simply follows the same steps as above, but using the Serfling bound (Lemma \ref{lemma:serf}) in place of the Hoeffding bound (Lemma \ref{lemma:hoeff}), finally setting the RHS to $\delta$ and solving for $k/1-f_k$ to complete the proof. 

\bibliographystyle{plain}

\begin{thebibliography}{10}

\bibitem{achlioptas2001database}
Dimitris Achlioptas.
\newblock Database-friendly random projections.
\newblock In {\em Proceedings of the twentieth ACM SIGMOD-SIGACT-SIGART
  symposium on Principles of database systems}, pages 274--281. ACM, 2001.

\bibitem{ailon2009}
Nir Ailon and Bernard Chazelle.
\newblock The fast johnson-lindenstrauss transform and approximate nearest
  neighbors.
\newblock {\em SIAM Journal on Computing}, 39(1):302--322, 2009.

\bibitem{ailon2009fast}
Nir Ailon and Edo Liberty.
\newblock Fast dimension reduction using rademacher series on dual bch codes.
\newblock {\em Discrete \& Computational Geometry}, 42(4):615, 2009.

\bibitem{arriaga1999algorithmic}
Rosa~I Arriaga and Santosh Vempala.
\newblock An algorithmic theory of learning: Robust concepts and random
  projection.
\newblock In {\em Foundations of Computer Science, 1999. 40th Annual Symposium
  on}, pages 616--623. IEEE, 1999.

\bibitem{bardenet2015concentration}
R{\'e}mi Bardenet, Odalric-Ambrym Maillard, et~al.
\newblock Concentration inequalities for sampling without replacement.
\newblock {\em Bernoulli}, 21(3):1361--1385, 2015.

\bibitem{bingham2001}
Ella Bingham and Heikki Mannila.
\newblock Random projection in dimensionality reduction: applications to image
  and text data.
\newblock In {\em Proceedings of the seventh ACM SIGKDD international
  conference on Knowledge discovery and data mining}, pages 245--250. ACM,
  2001.

\bibitem{dasgupta2003elementary}
Sanjoy Dasgupta and Anupam Gupta.
\newblock An elementary proof of a theorem of johnson and lindenstrauss.
\newblock {\em Random Structures \& Algorithms}, 22(1):60--65, 2003.

\bibitem{durrant2010compressed}
Robert~J. Durrant and Ata Kaban.
\newblock Compressed fisher linear discriminant analysis: classification of
  randomly projected data.
\newblock In {\em Proceedings of the 16th {ACM} {SIGKDD} International
  Conference on Knowledge Discovery and Data Mining, Washington, DC, USA, July
  25-28, 2010}, pages 1119--1128, 2010.

\bibitem{durrant2013sharp}
Robert~J. Durrant and Ata Kaban.
\newblock Sharp generalization error bounds for randomly-projected classifiers.
\newblock In {\em Proceedings of the 30th International Conference on Machine
  Learning, {ICML} 2013, Atlanta, GA, USA, 16-21 June 2013}, pages 693--701,
  2013.

\bibitem{durrant2014ACML}
Robert~J Durrant and Ata Kab{\'a}n.
\newblock Random projections as regularizers: learning a linear discriminant
  from fewer observations than dimensions.
\newblock {\em Machine Learning}, 99(2):257--286, 2014.

\bibitem{ho1998-2}
T.~K. Ho.
\newblock Random decision forest.
\newblock In {\em Proc. of the 3rd Int'l Conf. on Document Analysis and
  Recognition, Montreal, Canada, August}, pages 14--18, 1995.

\bibitem{ho1998}
Tin~Kam Ho.
\newblock The random subspace method for constructing decision forests.
\newblock {\em Pattern Analysis and Machine Intelligence, IEEE Transactions
  on}, 20(8):832--844, 1998.

\bibitem{hoeffding1963probability}
Wassily Hoeffding.
\newblock Probability inequalities for sums of bounded random variables.
\newblock {\em Journal of the American statistical association},
  58(301):13--30, 1963.

\bibitem{indyk2001algorithmic}
Piotr Indyk.
\newblock Algorithmic applications of low-distortion geometric embeddings.
\newblock In {\em focs}, volume~1, pages 10--33, 2001.

\bibitem{indyk1998approximate}
Piotr Indyk and Rajeev Motwani.
\newblock Approximate nearest neighbors: towards removing the curse of
  dimensionality.
\newblock In {\em Proceedings of the thirtieth annual ACM symposium on Theory
  of computing}, pages 604--613. ACM, 1998.

\bibitem{kaban2015improved}
Ata Kab{\'a}n.
\newblock Improved bounds on the dot product under random projection and random
  sign projection.
\newblock In {\em Proceedings of the 21th ACM SIGKDD International Conference
  on Knowledge Discovery and Data Mining}, pages 487--496. ACM, 2015.

\bibitem{kane2014}
Daniel~M Kane and Jelani Nelson.
\newblock Sparser johnson-lindenstrauss transforms.
\newblock {\em Journal of the ACM (JACM)}, 61(1):4, 2014.

\bibitem{kuncheva2010}
Ludmila~I Kuncheva, Juan~J Rodr{\'\i}guez, Catrin~O Plumpton, David~EJ Linden,
  and Stephen~J Johnston.
\newblock Random subspace ensembles for fmri classification.
\newblock {\em IEEE transactions on medical imaging}, 29(2):531--542, 2010.

\bibitem{lai2006}
Carmen Lai, Marcel~JT Reinders, and Lodewyk Wessels.
\newblock Random subspace method for multivariate feature selection.
\newblock {\em Pattern recognition letters}, 27(10):1067--1076, 2006.

\bibitem{larsen2014johnson}
Kasper~Green Larsen and Jelani Nelson.
\newblock The johnson-lindenstrauss lemma is optimal for linear dimensionality
  reduction.
\newblock {\em arXiv preprint arXiv:1411.2404}, 2014.

\bibitem{li2009}
Xiaoye Li and Hongyu Zhao.
\newblock Weighted random subspace method for high dimensional data
  classification.
\newblock {\em Statistics and its Interface}, 2(2):153, 2009.

\bibitem{liberty2009}
Edo Liberty and Steven~W Zucker.
\newblock The mailman algorithm: A note on matrix--vector multiplication.
\newblock {\em Information Processing Letters}, 109(3):179--182, 2009.

\bibitem{matousek2008}
Ji{\v{r}}{\'\i} Matou{\v{s}}ek.
\newblock On variants of the johnson--lindenstrauss lemma.
\newblock {\em Random Structures \& Algorithms}, 33(2):142--156, 2008.

\bibitem{meng2013scalable}
Xiangrui Meng.
\newblock Scalable simple random sampling and stratified sampling.
\newblock In {\em ICML (3)}, pages 531--539, 2013.

\bibitem{serfling1974probability}
Robert~J Serfling.
\newblock Probability inequalities for the sum in sampling without replacement.
\newblock {\em The Annals of Statistics}, pages 39--48, 1974.

\bibitem{skurichina2002}
Marina Skurichina and Robert~PW Duin.
\newblock Bagging, boosting and the random subspace method for linear
  classifiers.
\newblock {\em Pattern Analysis \& Applications}, 5(2):121--135, 2002.

\bibitem{tao2006}
Dacheng Tao, Xiaoou Tang, Xuelong Li, and Xindong Wu.
\newblock Asymmetric bagging and random subspace for support vector
  machines-based relevance feedback in image retrieval.
\newblock {\em IEEE transactions on pattern analysis and machine intelligence},
  28(7):1088--1099, 2006.

\bibitem{venkatasubramanian2011johnson}
Suresh Venkatasubramanian and Qiushi Wang.
\newblock The johnson-lindenstrauss transform: an empirical study.
\newblock In {\em Proceedings of the Meeting on Algorithm Engineering \&
  Expermiments}, pages 164--173. Society for Industrial and Applied
  Mathematics, 2011.

\bibitem{usc-sipi}
Allan~G Weber.
\newblock {USC - Signal and Image Processsing Institute}.
\newblock \texttt{http://sipi.usc.edu/database/}, 2006.
\newblock Accessed: 2016-08-30.

\end{thebibliography}

\end{document}